\documentclass[12pt]{article} 
\usepackage{amsmath,amsthm,amssymb,bm,mathrsfs,}
\usepackage{algorithm,algorithmic,graphicx,subfig}
\usepackage{caption,enumerate,natbib,listings,setspace}
\usepackage[OT1]{fontenc}
\usepackage[colorlinks,linkcolor=red,anchorcolor=blue,citecolor=blue]{hyperref}
\usepackage[margin=1in]{geometry}
\usepackage[protrusion=false,expansion=true]{microtype}
\usepackage[title]{appendix}
\usepackage{enumitem}
\theoremstyle{plain}

\let\hat\widehat

\newtheorem{example}{{\bf Example}}
\numberwithin{example}{section}

\newtheorem{lemma}{{\bf Lemma}}

\newtheorem{theorem}{{\bf Theorem}}

\newtheorem{definition}{{\bf Definition}}

\newtheorem{remark}{{\bf Remark}}

\title{\Large{\textbf{
Online Regularization towards \\}}
\Large{\textbf{Always-Valid High-Dimensional Dynamic Pricing}}} 

 \author
 {
 Chi-Hua Wang \thanks{Department of Statistics, UCLA, CA, 90095. Email: tsubasa3002101@gmail.com.},
 Zhanyu Wang\thanks{Department of Statistics,  Purdue University, IN 47907. Email: wang4094@purdue.edu}, 
 Will Wei Sun\thanks{Daniels School of Business, Purdue University, IN 47907. Email: sun244@purdue.edu.},
 and
 Guang Cheng\thanks{Department of Statistics, UCLA, CA, 90095. Email:guangcheng@ucla.edu}
 }

\date{}
\begin{document} 

\maketitle

\begin{abstract}
\noindent
Devising a dynamic pricing policy with always valid online statistical learning procedures is an important and as yet unresolved problem. Most existing dynamic pricing policies, which focus on the faithfulness of adopted customer choice models, exhibit a limited capability for adapting to the online uncertainty of learned statistical models during the pricing process. 
In this paper, we propose a novel approach for designing a dynamic pricing policy based on regularized online statistical learning with theoretical guarantees. The new approach overcomes the challenge of continuous monitoring of the online Lasso procedure and possesses several appealing properties. 
In particular, we make the decisive observation that the always-validity of pricing decisions builds and thrives on the \textit{online regularization} scheme. Our proposed online regularization scheme equips the proposed optimistic online regularized maximum likelihood pricing (\texttt{OORMLP}) pricing policy with three major advantages: encode market noise knowledge into pricing process optimism; empower online statistical learning with always-validity overall decision points; envelop prediction error process with time-uniform non-asymptotic oracle inequalities.  This type of non-asymptotic inference results allows us to design more sample-efficient and robust dynamic pricing algorithms in practice. 
In theory, the proposed \texttt{OORMLP} algorithm exploits the sparsity structure of high-dimensional models and secures a logarithmic regret in a decision horizon.  These theoretical advances are made possible by proposing an optimistic online Lasso procedure that resolves dynamic pricing problems at the \textit{process} level, based on a novel use of non-asymptotic martingale concentration. In experiments, we evaluate \texttt{OORMLP} in different synthetic and real pricing problem settings and demonstrate that \texttt{OORMLP} advances the state-of-the-art methods.
\end{abstract}


\newpage
\baselineskip=25pt 


\section{Introduction}
\label{sec:introduction}

With the growing availability and differentiation of digital products, modern online marketplaces present a unique challenge for dynamic pricing algorithms: they must customize pricing decisions for a diverse range of digital goods to the seller's customer database in an online environment. In response to such a unique challenge, \textit{online} training in modern dynamic pricing systems has increasingly included market knowledge and business insights, such as product features, marketing environment, and customer purchasing behavior. Indeed, dynamic pricing has been employed in a variety of services and businesses, including hospitality, tourism, entertainment, retail, energy, and public transportation \citep{boer2015surveys}, and has evolved into an integral part of revenue management in modern online service industries. 

A significant challenge of dynamic pricing in the modern digital economy is making customized pricing decisions for products, services, and solutions on the basis of item-level data. Besides, while most practical scenarios involve high-dimensional item-level data, only a small number of the observed features are typically decisive in the pricing decision process. In addition, high-dimensional dynamic pricing procedure has another layer of complexity: the entire pricing decision-making process is trained and learned from binary feedback. 
That is, pricing decision makers only observe and learn from the sale status for the price that was delivered, rather than learning from the true market value of the current item. To generate business insights on pricing mechanism, it is desirable to learn models that attributes to small number of decisive pricing factors to enhance explainablity of online learned market value model of products while maximizing the revenue.

Further, risk control of the online learned model on \textit{continuously monitor} dynamic pricing procedure is in emerging demand from industrial practice because the opportunity cost of lengthy pricing experiments is high and regrettable (\cite{johari2021always}). Indeed, it is desirable to detect the true product market value as quickly as possible or to abolish the running pricing experiment if the revenue improvement appears unpromising so that the scientist may test other available actions. Besides, optimizing the running time in advance is unfortunately impractical due to lacking knowledge on seeking revenue improvement and cost elasticity. In modern dynamic pricing practice, deployment of online statistical learning methodology turns out to be impeded by the such dynamic trade-off between maximum revenue improvement detection and minimum running time. Resolving such trade-off is a crucial advancement in statistical methodology for real-time data and persuades our investigation on the problem of \textit{continuous monitoring high-dimensional dynamic pricing problems}.

Continuous monitoring of high-dimensional dynamic pricing problem is a setting in which decision-makers seek to recover a sparse product market value model and maximize collected revenue (high-dimensional dynamic pricing), while the decision-makers are allowed to terminate the pricing algorithm whenever they wish, \textit{and} the result still maintains statistical validity (continuous monitoring). Such a setting arises naturally in industrial practice \citep{johari2021always} but remains challenging in the literature, preventing practitioners from effectively deploying high-dimensional statistical methodology effectively in modern online service industries.
Specifically, we consider a company that sells products to customers over a \textit{randomly stopped} time horizon. Each period, a new product is introduced, and the dynamic pricing algorithm is responsible for deciding its price. The pricing decision is based on the product feature and the historical pricing and sales data. Once the price is decided, the market either accepts or rejects the product, depending on whether the price is less than or more than the product's market value. The company has no idea what the market value of each product is, other than that it is a function in terms of the value of the product feature \citep{broder2012dynamic, keskin2014dynamic, javanmard2019dynamic}. 
Accordingly, the seller can utilize historical prices and sales data to infer market values for various product features and use those estimations to drive future pricing decisions. In general, one objective is to design a pricing algorithm that performs well in generating a small amount of worst-case regret.

Consequently, successful revenue management requires faithful product market value models and valid online statistical learning. Existing dynamic pricing studies focus on the faithfulness of adopted customer choice models \citep{myerson1981optimal, joskow2012dynamic, boer2015surveys, javanmard2019dynamic, mueller2019lowrank, nambiar2019dynamic, shah2019semi,ban2021personalized, javanmard2020multi}, but, unfortunately, this is insufficient: certain iterates within their online optimization process may violate pre-specified optimization constraints (for example, sparsity constraint) and thus deny the validity of ultimate pricing decisions. Such lack of validity haunts practitioners’ deployment of dynamic pricing systems and challenges scientists’ craftsmanship: \textit{how can one design an online regularization scheme to ensure the validity of online statistical learning uniformly among all decision points and secure low regret at the same time?} Specifically, we aim to deliver a regularization automation scheme based on learned-online market knowledge.

\subsection{Our contributions}

In this work, we make the decisive observation that the always-validity of pricing decisions builds on the \textit{online regularization} scheme. This insight is drawn from an elegant interplay between sparse online statistical learning and non-asymptotic martingale concentration, which is desirable to establish the always-validity of pricing decisions. Such interplay leads us to propose a novel online regularization scheme: we identify uncertainties surrounding learned product demand parameters and regularize them to ensure the feasibility of iterating over all decision points within the pre-specified confidence budget. In such a sense, a successful always-valid high-dimensional dynamic pricing algorithm design will always return valid pricing decisions with high probability. Hence, we regularize sparse online statistical learning by quantifying and offsetting uncertainties evolving within the estimation process.

We call this principle technical tool \textit{Optimistic Online LASSO}  (\texttt{OOLASSO}): a novel \textit{online regularization} scheme for online lasso. Based on it, we propose an optimistic online regularized maximum likelihood pricing (\texttt{OORMLP}) algorithm. The \texttt{OORMLP} enjoys three major advantages: encode market noise knowledge into pricing process optimism; empower online statistical learning with always-validity overall decision points; envelop estimation error process with time-uniform non-asymptotic concentration bounds. These properties ensure the validity and robustness of our algorithm in practical dynamic pricing problems. In theory, we establish (\texttt{OOLASSO}) a non-asymptotic time-uniform oracle inequality of our estimator. Such inequality is possible by our novel use of non-asymptotic martingale concentration inequalities \citep{maillard2019mathematics, howard2020time} 
to ensure the always-validity warranty under a user-specified confidence budget. Built upon this time-uniform oracle inequality, we further show that our \texttt{OORMLP} algorithm achieves a logarithm regret bound, which meets the information-theoretical lower bound in the literature (Theorem 5.1, \cite{javanmard2019dynamic}).
In the experiment, we evaluate the performance of \texttt{OORMLP} in both synthetic and real data set. The results back up our theoretical superiority of \texttt{OORMLP} algorithm in its robustness perspective against different demand uncertainties. Besides, we demonstrate how \texttt{OORMLP} utilizes the user-specified confidence budget into an online regularization scheme to trade off price exploration and exploitation to achieve a substantial regret reduction in finite time performance compared to \texttt{RMLP} (\cite{javanmard2019dynamic}).

In summary, our paper makes the following three major contributions. 
\begin{enumerate}[leftmargin=1mm, noitemsep]
\item Conceptually, we formulate the continuous monitoring of high-dimensional dynamic pricing problems. Our formulation bridges the high-dimension statistics literature in the Statistics community with continuous monitoring literature in the Operations Research community, opening a new venue for future studies on practical online statistical learning frameworks. 

\item Methodologically, we propose the \texttt{OORMLP} algorithm for continuous monitor high-dimensional dynamic pricing to ensure the pricing strategy is valid at any time. To our knowledge, this is the first high-dimensional dynamic pricing algorithm with an always-valid guarantee.

\item Theoretically, we establish time-uniform Lasso oracle inequalities on the estimation error process and further show a time-uniform logarithmic regret bound for our \texttt{OORMLP} algorithm. As a technical by-product, we develop \texttt{OOLASSO} to manage the optimism of online LASSO procedure via our novel use of non-asymptotic martingale concentration. 
\end{enumerate}


\subsection{Related literature}

Our work contributes to the learning-based dynamic pricing literature in problem formulation, to regularized online statistical learning in methodology, and to the growing literature of always valid online decision-making in theory. 

\textbf{Dynamic pricing with demand learning.}
Dynamic pricing with learning is a field of research that investigates pricing algorithms for situations when the demand function is unknown. 
Typically, the challenge is described as a form of the multiarmed bandit problem, with the arms being priced and the payoffs from the different arms being correlated, due to the measurements of demand assessed at different price points being correlated random variables. 
This includes parametric approaches \citep{broder2012dynamic, keskin2014dynamic, broder2012dynamic}, semi-parametric ones \citep{shah2019semi} as well as nonparametric ones \citep{fan2021policy, liu2022non, keskin2014dynamic}.
Beyond these studies, our work advances the problem formulation from finite to randomly stopped and possibly infinite horizon to meet the demand of continuous monitoring dynamic pricing in modern online service industrial practice. 



A more related line of work is contextual dynamic pricing, which can be categorized into three groups, with different emphasis on how the context plays roles in the price and products market demand or value. The first group of references \citep{qiang2016dynamic, nambiar2019dynamic, wang2021uncertainty} uses context $x$ as covariates of market demand. They assume the demand is observable and has a relationship with the offered price and the product context. In our work, we don't observe the demand but only the sale status of a product. The second group of references \citep{mao2018contextual, cohen2020feature} considers a noise-less contextual dynamic pricing, which captures the relationship between value and product context in a deterministic way. The third group of references \citep{javanmard2019dynamic, luo2021distribution, fan2022policy} considers a noisy linear valuation model, which is also the model used in our paper.




 


\textbf{Regularized online statistical learning.} 
In the past decade, regularized offline statistical learning methodology, including Ridge regression \citep{hoerl1970ridge} and Lasso regression \citep{tibshirani1996regression} and related high-dimensional literature \citep{buhlmann2011statistics, negahban2012unified, wainwright2019high}, have found their applications integral to the solution for various online machine learning task. The applications span across several different tasks including bandit algorithms design \citep{wang2020residual, wu2022residual}, online decision making \citep{bastani2020online, wang2020online, chen2021statistical_b, chen2021statistical_a, wang2022always}
and high-dimensional dynamic pricing \citep{javanmard2019dynamic, fan2021policy}. Indeed, these efforts inspired people to several proof concepts and elegant statistical frameworks for online machine learning tasks. However, the associated calibration scheme for regularization level in these prior efforts is typically designed for offline uncertainty (where the dataset is assumed given) but not online uncertainty (where the dataset is not given), leading to concerns about the validity of online-learned models and the consequent inference result. Beyond these studies, our work advances the methodology of regularized statistical learning from a constant level regularization for offline uncertainty to a process level regularization for online uncertainty, which we term \textit{online regularization}.

Such online regularization marks the key difference of our work compared to the \texttt{RMLP} in \cite{javanmard2019dynamic}, which also considered sparse learning in high-dimensional dynamic pricing. In practice, addressing the continuous monitoring high-dimensional dynamic pricing problems requires rethinking on the art of \texttt{RMLP} in the following three respects: (1) Rethink how to formulate the online uncertainty. In \texttt{RMLP}, the noise is assumed to be i.i.d, which does not capture the dependency nature between observations in the online setting. In contrast, we consider a martingale difference noise distribution, which is more suitable to quantify online uncertainty.
(2) Rethink the product feature sequence distribution. In \texttt{RMLP}, the product feature vectors are independently and identically sampled from a fixed distribution. In contrast, our framework allows a non-i.i.d. or general feature distribution. 
(3) Rethink the regularization level sequence. \texttt{RMLP} considers episode updates and requires resetting the algorithm. In contrast, our design of regularization does not need to reset the algorithm, hence is more sample efficient. Moreover, our regularization design mechanism also includes product context uncertainty and confidence budget to better balance the tradeoff between online uncertainty and online estimation error. 
See Remark \ref{rm:compare_RMLP} and Remark \ref{rm:alg_design} for more detailed comparisons of these two methods.

\textbf{Always-valid online decision making.} Always-valid online decision making is an emerging field of studies in the last half decade \citep{johari2015always, zhao2016adaptive, johari2021always}. Such emergence is a response of surging demand from modern online service industrial practice since the opportunity cost of lengthy online experiments is high and regrettable \citep{johari2021always}. Indeed, it is preferable to determine the real impact as fast as feasible or to terminate the ongoing experiment if the result looks unpromising, allowing the scientist to try other activities. Additionally, adjusting the runtime length in advance is unfeasible due to a lack of knowledge about the amount of the seeking impact and cost elasticity. Consequently, in modern online service practice, such dynamic trade-offs between greatest effect detection and shortest running time constrain the implementation of online statistical learning methodologies. Our work makes a first advance on the theory of always-valid online decision making into the high-dimensional dynamic pricing problems.

\section{High-dimensional dynamic pricing problems}\label{sec:2_hd_pricing}

This section defines the high-dimensional dynamic pricing problems. Section \ref{subsec:alg_design} provides a five-step general design and essential elements of dynamic pricing algorithms. Section \ref{sec:choice_model} provides our statistical framework for the market value of the product. Section \ref{subsec:pricing_function} provides our presumption on the implemented pricing function.

\subsection{A general design of dynamic pricing algorithms}
\label{subsec:alg_design}

\begin{figure}[h]
\centering
\includegraphics[width=\textwidth]{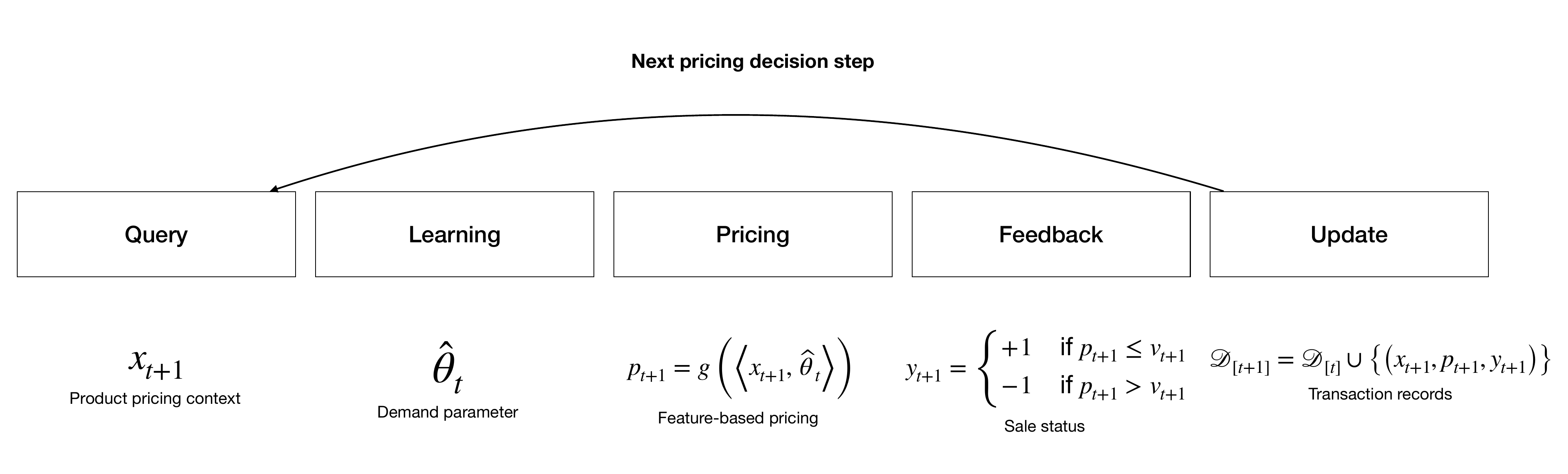}
\caption{A general design of dynamic pricing algorithms}
\label{fig:DP_design}
\end{figure}

In a dynamic pricing problem with decision horizon $T$, the agent is required to determine total $T$ prices at decision points $1,2, \cdots , T$. Here $T$ is an \textit{unknown integer-valued random variable} and its realization is determined by an unknown terminating rule from a decision maker.
At a decision point $t \in [T]$, a customer in the market selects a product with context $x_{t}$ from a $d$-dimensional unit sphere $\mathcal{X}=\{x\in \mathbb{R}^{d}:\|x\|_{\infty}\le 1 
\}$. The agent receives a pricing query for $x_{t}$, and her goal is to choose a posted price $p_{t} \in \mathbb{R}$ to maximize the revenue. The market value $v_{t}$ of product $x_{t}$ is unknown. 
After posting a price $p_{t}$, the customer decides whether to purchase
the product based on market value $v_{t}$. The market value $v_{t}$ is not observable to the agent, but only a binary-valued sale status variable $y_{t} \in \{-1, +1\}$. If $p_{t}\le v_{t}$, a sale occurs and the seller collects a revenue $p_{t}$ and $y_{t} = +1$; otherwise, no sale occurs and no revenue is received and $y_{t} = -1$. Formally, 
\begin{equation}\label{eq:sale_status}
    y_{t} = 
\begin{cases}
    +1 &\text{if } p_{t} \le v_{t}\\
    -1 &\text{if } p_{t} > v_{t}
\end{cases}
\end{equation}
The seller's objective is to develop a pricing policy that maximizes revenue received.

Figure \ref{fig:DP_design} briefly summarizes a general design of dynamic pricing algorithms for revenue maximization via illustration of the five steps in a single decision step. In particular, 
at each decision point $t+1$, the agent
\begin{enumerate}[leftmargin=1mm, noitemsep]
\item \textbf{Query}: The algorithm receives a query for pricing on the product with high-dimensional context vector $x_{t+1} \in \mathcal{X}$. 
\item \textbf{Learning}: The algorithm learns a demand parameter estimate $\widehat{\theta}_{t} \in \Omega$ based on up-to-time $t$ transaction records $\mathcal{D}_{[t]}=\{(x_{s}, p_{s}, y_{s})\}_{s=1}^{t}$ to predict market value $v_{t+1}$ of product $x_{t+1}$. 
\item \textbf{Pricing}: The algorithm posts a revenue-maximizing price $p_{t+1} = g(\widehat{\theta}_{t}; x_{t+1})$ with a user-specified pricing function $g$. 
\item \textbf{Feedback}: The algorithm receives a sale status $y_{t+1}$, based on the product's sale price $p_{t+1}$. 
\item \textbf{Update}: The algorithm updates the transaction records $\mathcal{D}_{[t+1]}
=\mathcal{D}_{[t]}\cup\{(x_{t+1}, p_{t+1}, y_{t+1})\}$.
\end{enumerate}

Building upon the above general design of dynamic pricing algorithms, our goal is to provide an online statistical learning framework that fulfills three desiderata--sparse learning, always-validity, and revenue-maximization--that outlined in Section \ref{sec:good_pricing_policy} to resolve high-dimensional dynamic pricing problems in continuous monitoring setting. The resulting dynamic pricing algorithms and the statistical learning framework are established in Section \ref{sec:product} and their formal fulfillment to the three desiderata are elaborated in Section \ref{sec:qualities}.

\subsection{Product market value model}\label{sec:choice_model}

Our statistical framework for market value $v_{t}$ of product $x_{t}$ consists of three parts: the market value model $v_{t}|x_{t}$, the target demand parameter $\theta_0$ and the martingale difference noise process $\{\eta_{t}\}_{t=1}^{T}$. First, 
we model market value $v_{t}$ of the product as a linear function of the observable product covariate $x_{t}$; formally 
\begin{equation}\label{eq:Will_to_Pay_linear_model}
    v_{t} = \langle \theta_0, x_{t} \rangle + \eta_{t}.
\end{equation}

Second, the unknown parameter $\theta_0$ is the {target demand parameter} that characterizes the demand profile of customers' behaviors. Parallel to  high-dimensional dynamic pricing literature \citep{javanmard2019dynamic}, we consider a structured feasible parameters set $\Omega$ in which 
$\theta_0$ is high-dimensional and sparse; formally, for user-specified constants $s_0$ and $W$, the feasible parameters set $\Omega$ is defined as
\begin{equation}\label{eq:feasible_parameter}
\Omega = \{\theta \in \mathbb{R}^{d}: \|\theta\|_0 \le s_0, \|\theta\|_1 \le W\}.
\end{equation}

Third, the noise process $\{\eta_{t}\}_{t=1}^{T}$ in $(\ref{eq:Will_to_Pay_linear_model})$ accounts for unmeasured context and random noises. Notably, we consider a more general and practical dependent noise process drawn from a martingale difference sequence that is adapted to current transaction records. That is, with respect to the $\sigma$-field 
\begin{equation}\label{eq:history}
\mathcal{H}_{t-1}=\sigma(x_{1}, p_{1}, y_{1}, \cdots, x_{t-1}, p_{t-1}, y_{t-1}, x_t, p_t)
\end{equation}
generated by all transaction records before $y_t$ is observed, the noise process $\eta_{t}$ satisfies  $\mathbb{E}[\eta_{t}|\mathcal{H}_{t-1}] = 0$ for all $t\in[T]$. Our dependent noise process relaxes the i.i.d. assumption considered in \cite{javanmard2019dynamic}. The conditional distribution of $\eta_{t}|\mathcal{H}_{t-1}$ is assumed to be log-concave in this paper. Many common probability distributions such as normal, logistic, uniform, exponential, Laplace, and bounded distributions are log-concave \citep{wellner2012log}. In particular, we define the `steepness` of a function $F_{\eta_{t}|\mathcal{H}_{t-1}}(\cdot)$ as
\begin{subequations}
\begin{equation}\label{eq:steepness}
u_{W, t} \equiv \sup _{|x| \leq 3 W}\left\{\max \left\{\log ^{\prime} F_{\eta_{t} \mid \mathcal{H}_{t-1}}(x),-\log ^{\prime}\left(1-F_{\eta_{t} \mid \mathcal{H}_{t-1}}(x)\right)\right\}\right\}
\end{equation}
and also define the 'flatness' of function $F_{\eta_{t}|\mathcal{H}_{t-1}}(\cdot)$ as 
\begin{equation}\label{eq:flatness}
l_{W,t} \equiv \inf _{|x| \leq 3 W}\left\{\min \left\{-\log ^{\prime \prime} F_{\eta_{t} \mid \mathcal{H}_{t-1}}(x),-\log ^{\prime \prime}\left(1-F_{\eta_{t} \mid \mathcal{H}_{t-1}}(x)\right)\right\}\right\}.
\end{equation}
\end{subequations}
In addition, we define the \textit{maximal steepness} to be the constant $u_{W}=\max_{t\in[T]}u_{W,t}$ and the \textit{minimal flatness} to be the constant  $l_{W}=\min_{t \in [T]}l_{W,t}$.

The above statistical framework of product market value induces a probabilistic model for the sale status process $\{y_{t}\}_{t=1}^{T}$. The sale status process denotes a trajectory of customer transaction decisions with respect to the corresponding pricing sequence $\{p_{t}\}_{t=1}^{T}$ and product sequence $\{x_{t}\}_{t=1}^{T}$. In particular, given the definition of sale status \eqref{eq:sale_status} and the market value model \eqref{eq:Will_to_Pay_linear_model}, the sale status process  $\{y_{t}\}_{t=1}^{T}$ is generated from the following probabilistic model:
\begin{equation}\label{eq:sale_status_stoc_model}
    \mathbb{P}_{\theta_0}(y_{t}|\mathcal{H}_{t-1})
    =
     \begin{cases}
     1-F_{\eta_{t}|\mathcal{H}_{t-1}}(p_{t}-\langle \theta_0, x_{t} \rangle)
      &\text{ if } y_{t} = +1, \\
    F_{\eta_{t}|\mathcal{H}_{t-1}}(p_{t}-\langle \theta_0, x_{t} \rangle)
     & \text{ if } y_{t}=-1
     ,
    \end{cases}
\end{equation}
where $F_{\eta_{t}|\mathcal{H}_{t-1}}(\cdot)$ denotes the conditional distribution of noise $\eta_{t}$ given $\mathcal{H}_{t-1}$.

\subsection{Pricing function}
\label{subsec:pricing_function}

Our framework allows a flexible pricing function $g$ used at Step 3 of the pricing algorithm design (Figure \ref{fig:DP_design}). Such a feature is standard in industrial practice to provide flexible deployment of dynamic pricing algorithms \citep{johari2021always}.  
We assume the pricing function $g$ is a $L$-Lipschitz continuous function for some Lipschitz constant $L\le 1$, which is satisfied by the common pricing function choice in the literature, given in Example \ref{eg:rev_max_price_function}.

\begin{example} \label{eg:rev_max_price_function}
To maximize the expected revenue, it is shown in auction theory \citep{myerson1981optimal, javanmard2019dynamic}, the revenue-maximizing price 
$p^{*}(x_t)=\arg\max_{p}\{p(1-F_{\eta_{t}|\mathcal{H}_{t-1}}(p-\langle \theta_0, x_{t}))\}.$
The first order conditions says that the optimal posted price $p_{t}^{*}=p^{*}(x_{t})$ satisfy
$$p_{t}^{*}=\frac{1-F_{\eta_{t}|\mathcal{H}_{t-1}}\left(p_{t}^{*}-\langle\theta_{0}, x_{t}\rangle\right)}{f_{\eta_{t}|\mathcal{H}_{t-1}}\left(p_{t}^{*}-\langle\theta_{0}, x_{t}\rangle \right)}=p_{t}^{*}-\langle\theta_{0}, x_{t}\rangle-\phi_{t}(p_{t}^{*}-\langle\theta_{0}, x_{t}\rangle)$$
by letting $\phi_{t}(v)\equiv v - \frac{1-F_{\eta_{t}|\mathcal{H}_{t-1}}(v)}{f_{\eta_{t}|\mathcal{H}_{t-1}}(v)}$.
That is, 
$\langle\theta_{0}, x_{t}\rangle+\phi_{t}(p_{t}^{*}-\langle\theta_{0}, x_{t}\rangle)=0$
and hence
$p_{t}^{*}= \langle\theta_{0}, x_{t}\rangle + (\phi_{t})^{-1}(-\langle\theta_{0}, x_{t}\rangle)=g_{t}(\langle\theta_{0}, x_{t}\rangle)$.
So the pricing function has the closed form 
\begin{equation}\label{eq:opt_price_function}
g_{t}(v)\equiv v + (\phi_{t})^{-1}(-v),
\end{equation}
where $\phi_{t}(v) \equiv v - (1-F_{\eta_{t}|\mathcal{H}_{t-1}}(v))/f_{\eta_{t}|\mathcal{H}_{t-1}}(v)$ is known as a \textit{virtual valuation} function. By Lemma S7.4, the pricing function $g_{t}$ is 1-Lipschitz continuous. 
\end{example}

\section{Evaluating dynamic pricing policy}
\label{sec:good_pricing_policy}

In this section, we elaborate on what makes a good dynamic policy. 
Our goal is to design a pricing policy $\pi$ that offers the price $p_{t}(\pi)$ for the product $x_{t}$ in order to (i) \textbf{learn} the true demand parameter $\theta_0$ to inform seller about the underlying product market value model \eqref{eq:Will_to_Pay_linear_model}, 
(ii) \textbf{continuously monitor} the estimation error of the estimated demand parameter, and (iii) \textbf{optimize} the posted price to maximize the expected revenue. In order for the policy $\pi$ to fulfill the learning and optimizing tasks, it must satisfy the following desiderata: (A) it should return a sparse demand parameter estimate to enhance the explainability of the pricing mechanism and product market value,  
(B) it should be able to \textit{adapt the online uncertainty} of product market value model \eqref{eq:Will_to_Pay_linear_model} to obtain \textit{always-valid} statistical error bounds, and (C) it should be \textit{revenue-maximized}, i.e., the difference between posted price $p_{t}(\pi)$ and the oracle price $\pi^{*}_{t}$ should be small. Consequently, it's critical to establish an effective strategy that strikes a balance between exploration (gathering data for learning parameters) and exploitation (offering optimal pricing based on learned parameters).

Having outlined the desiderata for our sought-after pricing policy, we now propose three properties of the online statistical learning framework that should be encoded in the adopted pricing policy.  
These properties are:
\begin{enumerate}[leftmargin=1mm, noitemsep]
\item[(A)] Sparse Learning:
the learned demand parameter identifies the subset of decisive pricing features to enhance the explainability of the learned market value model. (Section \ref{subsec:online_lasso})
\item[(B)]  Always-Validity: the estimation error of the online learned market value model remains statistical validity even when the pricing algorithm is terminated randomly.
(Section \ref{subsec:always_valid_bound})
\item[(C)]  Revenue-Maximization:
the collective revenue 
is comparable to the revenue of the oracle pricing policy which knows the true demand parameter. (Section \ref{subsec:regret})
\end{enumerate}

\subsection{Online Lasso procedure towards sparse learning}\label{subsec:online_lasso}

To achieve the first desiderata on learning sparse demand parameter estimate, we adopt the online Lasso procedure, defined as follows.
\begin{definition}
\label{def:online_lasso_procedure}
We define the \textbf{online Lasso procedure} as follows:
\begin{enumerate}[leftmargin=1mm, noitemsep]
\item At a decision point $t$, the agent calculates the negative log-likelihood function $\mathcal{L}(\theta; \mathcal{D}_{[t]})$ of a model parameter $\theta$ and up-to-time $t$ transaction records $\mathcal{D}_{[t]}$ as 
\begin{subequations} 
\begin{equation}\label{eq:self_info}
\mathcal{L}_{t}(\theta)\equiv \mathcal{L}(\theta; \mathcal{D}_{[t]}) = t^{-1}\sum_{s=1}^{t}\log(1/\mathbb{P}_{\theta}(y_{s}|\mathcal{H}_{s-1})).
\end{equation}
The probability  $\mathbb{P}_{\theta}(y_{s}|\mathcal{H}_{s-1})$ is from the Bernoulli model \eqref{eq:sale_status_stoc_model} of the sale status process $\{y_{t}\}_{t=1}^{T}$; that is, with $u_{t}(\theta)\equiv p_{t}-\langle \theta, x_{t} \rangle $, 
$\log(1/\mathbb{P}_{\theta}(y_{s}|\mathcal{H}_{s-1}))
=
\mathbb{I}\left(y_{t}=1\right) \log \left(1/(1-F_{\eta_{t} \mid \mathcal{H}_{t-1}}\left(u_{t}(\theta)\right))\right)+\mathbb{I}\left(y_{t}=-1\right) \log \left(1/F_{\eta_{t} \mid \mathcal{H}_{t-1}}\left(u_{t}(\theta)\right)\right).$
\item The algorithm penalizes the loss
$\mathcal{L}_{t}(\theta)$ by the $l_{1}$-norm penalty at regularization level $\lambda_{t}>0$. In particular, at decision point $t$, the algorithm learns an estimator $\widehat{\theta}_{t}$ by solving the $\ell_{1}$-regularized quadratic program
\begin{equation}\label{eq:current_LASSO_program}
    \widehat{\theta}_{t}
    \equiv \arg\min_{
    \|\theta\|_1 \le W} \bigg\{
    \mathcal{L}_{t}(\theta) 
    +
    \lambda_{t}
    \|\theta\|_1
    \bigg\}.
\end{equation}
\end{subequations}
\item 
Repeating the above Lasso procedure at each decision point $t=1,2,\cdots, T$, with a regularization level sequence $\{\lambda_{t}\}_{t=1}^{T}$, the agent thus learns at the decision horizon $T$  an estimation sequence: $\widehat{\theta}_{1}, \widehat{\theta}_{2}, \cdots, \widehat{\theta}_{T}$. 
\end{enumerate}
\end{definition}

The online Lasso procedure (Definition \ref{def:online_lasso_procedure}) delivers a statistical learning framework for online sparse learning. In practice, given a regularization level sequence $\{\lambda_{t}\}_{t=1}^{T}$, the online Lasso procedure returns a sequence of constrained estimators $\{\widehat{\theta}_{1}, \widehat{\theta}_{2}, \cdots, \widehat{\theta}_{T}\}$ towards learning a sparse demand parameter estimate in the product market value model \eqref{eq:Will_to_Pay_linear_model}.
Indeed, such online Lasso procedures benefit the interpretability of the resulting product market value model and the explainability of the pricing mechanism. 

However, the benefit of the online Lasso procedure may be blocked by an improper choice of regularization level sequences $\{\lambda_{t}\}_{t=1}^{T}$. 
As well-recognized in the high-dimensional statistics literature, different regularization level sequences $\{\lambda_{t}\}_{t=1}^{T}$ lead to different properties of resulting constrained estimators sequence $\{\widehat{\theta}_{1}, \widehat{\theta}_{2}, \cdots, \widehat{\theta}_{T}\}$. As far as the continuous monitoring dynamic pricing concerns, the fundamental challenge is how to choose the regularization level $\lambda_{t}$ in Lasso program \eqref{eq:current_LASSO_program} at a process level, i.e. for every decision step $t$ from 1 to the random decision horizon $T$. Section \ref{sec:oolasso} contributes the key observation that the online Lasso procedure builds and thrives on online regularization scheme design to calibrate online uncertainty during the pricing process.

\subsection{Always valid estimation error bound process} \label{subsec:always_valid_bound}

To achieve the second desiderata on always-valid online statistical learning, we introduce the concept of always-valid estimation error bound process, defined as follows :
\begin{definition}\label{def:always_valid_error_bound}
Given any (possible unbounded) stopping time $T$ with respect to historical filtration $\{\mathcal{H}_{t}\}_{t=0}^{T}$ (defined at \eqref{eq:history}). A sequence of constant real number $\{r_{t}\}_{t=1}^{T}$ is an \textbf{always valid estimation error bound process} of the estimator sequence $\{\widehat{\theta}_{1}, \widehat{\theta}_{2}, \cdots, \widehat{\theta}_{T}\}$ with confidence budget $\alpha$ if it holds that
\begin{equation}\label{eq:AlwaysValid_RiskBound}
\mathbb{P}_{\theta_0}\bigg(\exists t \in 
[T]: \|\hat{\theta}_{t}-\theta_0\|_2 > r_{t}\bigg) \le \alpha.
\end{equation}
\end{definition}

The always valid estimation error bound process (Definition \ref{def:always_valid_error_bound}) serves as a principal theoretical tool for online service industrial practice in the continuously monitoring risk control of adopted online statistical learning procedures. For an online learned estimator sequence $\{\widehat{\theta}_{1}, \widehat{\theta}_{2}, \cdots, \widehat{\theta}_{T}\}$, the corresponding error bound process $\{r_{1}, r_{2}, \cdots, r_{T}\}$ collectively gives a time-uniform control on the estimation error sequence $\{\|\widehat{\theta}_{1}-\theta_0\|_{2}, \|\widehat{\theta}_{2}-\theta_0\|_{2}, \cdots, \|\widehat{\theta}_{T}-\theta_0\|_{2}\}$ such that the probability of out-of-control is at most at the level of user pre-specified confidence budget $\alpha$. Such time-uniform risk control allows users to terminate the dynamic pricing algorithm whenever they wish, \textit{and} the result still maintains statistical validity. 

Establishing such an always valid error-bound process, however, is technically challenging and far from understood in the literature. The reason is that, while estimation error bound result for fixed sample size Lasso regression had been systematically studied in the literature and inspired people for an elegant theoretical framework, they focused on offline uncertainty (the whole dataset is given) instead of online uncertainty (the dataset is not given and is observed on the fly). Consequently, the classical method of high-dimensional statistics literature fails to meet the challenge of online statistical learning with continuous monitoring demanded in the modern online service industry. Section \ref{sec:time_unif_LASSO_oracle} contributes a key theoretical result on the always validity of online Lasso procedure (Definition \ref{def:online_lasso_procedure}).

\subsection{Regret of a dynamic pricing policy}
\label{subsec:regret}

To achieve the third desiderata of revenue maximization, we define the notion of regret.

\begin{definition}
The regret of a dynamic pricing policy $\pi$ up to decision $T$ is defined as
\label{def:regret}
\begin{equation}\label{eq:regret_def}
\textbf{Regret}_{\mathcal{\pi}}(T) \equiv \max_{\theta_0 \in \Omega}
\mathbb{E}\bigg[ 
\sum_{t=1}^{T}
\big( 
r_{t}(p_{t}^{*})-r_{t}(p_{t}(\pi))
\big)
\bigg],
\end{equation}
where $r_{t}(p)\equiv pI(v_{t} \ge p)$ is the expected revenue of the product $x_{t}$ with the posted price $p$. The expectation is taken with respect to the noise $\eta_{t}$ and product context $x_{t}$, and $p_{t}(\pi)$ denotes the price offered at decision step $t$ by following policy $\pi$. 
\end{definition}

Definition \ref{def:regret} benchmarks the performance of a dynamic pricing policy $\pi$ that determines posted prices $\{p_{t}\}_{t=1}^{T}$ to the corresponding 'oracle pricing policy', which exploits knowledge of the true demand parameter $\theta_0$ and proposes the price $p_{t}^{*}=g(\langle \theta_0, x_{t} \rangle )$ for the product of context $x_{t}$, where $g(\cdot)$ is a user-specified pricing function.
In Example \ref{eg:rev_max_price_function}, the optimal price $p_{t}^{*}$ is the price that maximizes the expected revenue. Formally, we consider the goal of maximizing revenue as minimizing the maximum regret at Definition \ref{def:regret}. As pursued as the third desiderata of pricing policy, the goal is to design an online statistical learning procedure such that the regret \eqref{eq:regret_def} is small.

\section{The \texttt{OORMLP} algorithm and \texttt{OOLASSO} procedure}
\label{sec:product}

This section establishes our pricing policy design that achieves the three desiderata discussed in Section \ref{sec:good_pricing_policy}.
We first propose the Optimistic Online Regularized Maximum Likelihood Pricing (\texttt{OORMLP}) algorithm (Algorithm \ref{alg:OORMMLP}) as the desirable dynamic pricing policy at Section  \ref{sec:OORMLP}. Then we elaborate our novel Optimistic Online Lasso procedure (\texttt{OOLASSO}) towards always valid online statistical learning at Section \ref{sec:OOLASSO_procedure}.

\subsection{\texttt{OORMLP} algorithm}
\label{sec:OORMLP}

In this section, we present the proposed dynamic pricing policy at Algorithm \ref{alg:OORMMLP}. The presentation follows the general design of dynamic pricing algorithms in Figure \ref{fig:DP_design}. In particular, at decision point $t+1$, the agent learns the demand parameter estimator $\widehat{\theta}_{t}$ based on the current transaction records $\mathcal{D}_{[t]}$ via Lasso regression in \eqref{eq:current_LASSO_program} at regularization level $\lambda_{t}$ specified in the optimistic online regularization scheme \eqref{eq:main_online_regularization}. In addition, both the sample covariance matrix $\widehat{\Sigma}_{[t]}$ and the online regularization sequence $\{\lambda_{t}\}_{t=1}^{T}$ in \eqref{eq:main_online_regularization} can be incrementally updated: at each decision point $t$, 
$$\widehat{\Sigma}_{[t]} \gets t^{-1}\left[(t-1)\widehat{\Sigma}_{[t-1]}+x_{t}x_{t}^\top\right];~~
\lambda_{t} \gets \lambda_{t-1}\sqrt{(1-t^{-1})\|\widehat{\Sigma}_{[t]}\|_{\infty}/\|\widehat{\Sigma}_{[t-1]}\|_{\infty}}.$$
Such property allows an efficient online implementation in the experiments.

\begin{algorithm}[t]
  \caption{Optimistic Online Regularized Maximum Likelihood Pricing (\texttt{OORMLP})}
  \label{alg:OORMMLP}
  \begin{algorithmic}[1]
   \REQUIRE Steepness of market noise $u_{W}$, pricing function $g(\cdot)$ and confidence budget $\alpha$.
   \STATE \textit{Initialization}: Receive product context $x_1$.
   Post price $p_1$.
   Receive sale status $y_1$.
   \STATE 
   $\mathcal{D}_{[1]}\gets \{(x_1, p_1, y_1)\}$;
   $\widehat{\Sigma}_{[1]}\gets x_1x_1^\top$;
   $\lambda_{1} \gets 4u_{W} \sqrt{ 2 \|\text{diag}(\widehat{\Sigma}_{[1]})\|_{\infty}\ln(2d/\alpha) }.$  
   \FOR{$t = 2, \dots, [T]$}
   \STATE \textbf{1.Query:} Receive product context $x_{t}$.
   \STATE \textbf{2.Learning:} Update the sample covariance matrix and regularization level:
   \begin{subequations}
   \begin{align}
   \widehat{\Sigma}_{[t]} &\gets t^{-1}\left[(t-1)\widehat{\Sigma}_{[t-1]}+x_{t}x_{t}^\top\right],\\
   \lambda_{t} &\gets \lambda_{t-1}\sqrt{(1-t^{-1})\|\widehat{\Sigma}_{[t]}\|_{\infty}/\|\widehat{\Sigma}_{[t-1]}\|_{\infty}};
   \end{align}
   \end{subequations}
   \vspace{-5mm}
   \STATE Update the estimate
   \begin{equation}
   \widehat{\theta}_{t-1} \gets 
   \arg\min_{
   \|\theta\|_1 \le W
   } \left\{
    \mathcal{L}_{t-1}(\theta)
    +
    \lambda_{t-1}
    \|\theta\|_1\right\}.
   \end{equation}
   \STATE \textbf{3.Pricing:} Post price $p_{t} \gets g\left(\left\langle \widehat{\theta}_{t-1}, 
    x_{t}\right\rangle\right)$. 
   \STATE \textbf{4.Feedback:} Receive sale status $y_{t}$. 
   \STATE \textbf{5.Update:} $\mathcal{D}_{[t]}\gets \mathcal{D}_{[t-1]}\cup\{(x_t, p_t, y_t)\}.$
   \ENDFOR 
  \end{algorithmic}
\end{algorithm}

\subsection{Optimistic online lasso procedure} \label{sec:OOLASSO_procedure}

Here, we elaborate our novel approach to construct a learning process  $\widehat{\theta}_{1}, \widehat{\theta}_{2}, \cdots, \widehat{\theta}_{T}$ for the target demand parameter $\theta_0$ based on transaction records $\mathcal{D}_{[t]}=\{(x_s, p_s, y_s)\}_{s=1}^{t}$ with optimism in the face of online uncertainty during the pricing process.

\begin{definition}\label{def:opt_online_reg}
An online Lasso procedure (Definition \ref{def:online_lasso_procedure}) is \textbf{optimistic} if the regularization sequence $\{\lambda_{t}\}_{t=1}^{T}$ is specified by the following \textbf{optimistic online regularization scheme}:
\begin{equation}\label{eq:main_online_regularization}
    \lambda_{t}(\alpha) \equiv 4u_{W} \sqrt{2\cdot t^{-1}  \|\text{diag}(\widehat{\Sigma}_{[t]})\|_{\infty}  \ln(2d/\alpha) }.
\end{equation}
\end{definition}

Definition \ref{def:opt_online_reg} presents our novel regularization scheme for regulating online uncertainty during the dynamic pricing process. 
The reason we call \eqref{eq:main_online_regularization} optimistic is that it regularizes the online LASSO procedure with optimism in the face of both demand uncertainty and product feature uncertainty during the dynamic pricing process, given a specified confidence budget $\alpha$. Three factors contribute to the regularization level $\lambda_{t}(\alpha)$. First, the constant $u_{W}$ is the \textit{maximal steepness} of noise process \eqref{eq:steepness}
and represents our prior knowledge of demand uncertainty. Second, the empirical covariance matrix $\widehat{\Sigma}_{[t]}=t^{-1}\sum_{s=1}^{t}x_{s}x_{s}^\top$ characterizes the uncertainty of up-to-now product context sequence. Third, the constant $\alpha$ stands for the user-pre-specified confidence budget for the always-validity of implemented online LASSO procedure. These factors collectively express optimism in the face of online uncertainty during the dynamic pricing process and are the foundation to fulfill the three desiderata we pursued in Section \ref{sec:good_pricing_policy}. Consequently, we adopt the optimistic online regularization scheme \eqref{eq:main_online_regularization} to design \texttt{OORMLP} algorithm (Algorithm \ref{alg:OORMMLP}) to enjoy three desiderata-sparse learning, always-validity and revenue-maximization-on resulting dynamic pricing policy.

\begin{remark}\label{rm:compare_RMLP}(Regularization comparison to \texttt{RMLP} in \cite{javanmard2019dynamic}) 
The relation between our regularization scheme and the one in \texttt{RMLP} is $$\lambda_{t, \texttt{OORMLP}}(\alpha)=\lambda_{t, \texttt{RMLP}}  \sqrt{2 \frac{\log _{2}(t)}{t}} \sqrt{\frac{\log (2 d / \alpha)}{\log (d)}}\sqrt{\|\text{diag}(\widehat{\Sigma}_{[t]})\|_{\infty}}.$$ 
The relation above indicates that, while \texttt{RMLP} do not, \texttt{OORMLP} includes in the regularization level the uncertainty arising from context sequence ($\|\text{diag}(\widehat{\Sigma}_{[t]})\|_{\infty}$). The root reason why \texttt{RMLP} does not take $\|\text{diag}(\widehat{\Sigma}_{[t]})\|_{\infty}$ into their regularization design is due to their assumption on the independent identical distributed property on the product sequence. When the distribution of product sequence deviates from such i.i.d. assumption, the regularization level in \texttt{RMLP} is improper to account for the effective noise process in the LASSO procedure. Our regularization design takes the context sequence uncertainty into account, which leads to the robustness of \texttt{OORMLP} in the experiments. 
\end{remark}


\section{Always-validity and regret analysis}
\label{sec:qualities}

This section elaborates on formal guarantees of three qualities of our \texttt{OORMLP} algorithm and \texttt{OOLASSO} procedure. Section \ref{sec:oolasso} demystifies the design principle behind our optimistic online regularization scheme, formally achieving the first desiderata: sparse learning. Section \ref{sec:time_unif_LASSO_oracle} establishes the time-uniform Lasso oracle inequality (Theorem \ref{thm:always_oracle}), formally achieving the second desiderata: always-validity. 
Section \ref{subsec:regret_ana} present regret analysis (Theorem \ref{thm:regret_analysis}) of our \texttt{OORMLP} pricing policy, formally achieving the third desiderata: revenue-maximizing.

\subsection{Optimistic online regularization scheme}\label{sec:oolasso}

This section demystifies the optimistic online regularization scheme \eqref{eq:main_online_regularization} as a formal guarantee of the sparse learning of our online statistical learning framework.

\subsubsection{Basic design principle}\label{subsec:basic_design_regu}

We now explain the design principle of  the regularization sequence $\{\lambda_{t}\}_{t=1}^{T}$ for the optimistic online regularization scheme at \eqref{eq:main_online_regularization}. In principle, our goal is to design a regularization sequence $\{\lambda_{t}\}_{t=1}^{T}$ that warrants the online LASSO procedure (Definition \ref{def:online_lasso_procedure}) with always-validity by constructing an always valid estimation error bound process (Definition \ref{def:always_valid_error_bound}).
Intuitively, the optimal choice of the sequence is an outcome of the bias-and-variance trade-off. Bias arises as a shrinkage effect from $l_1$-regularizer and grows as $\lambda_{t}$ increases. Besides, $l_1$-regularizer offsets fluctuations in the score function process $\{\nabla \mathcal{L}_{t}(\theta)\}_{t=1}^{T}$. Hence, an optimal choice of $\{\lambda_{t}\}_{t=1}^{T}$ is the smallest \textit{envelop} that is large enough and \textit{always} controls score fluctuations during the whole pricing process.

To obtain an always valid estimation error bound process of the online LASSO procedure \eqref{eq:current_LASSO_program}, 
we generalize standard guidance from high-dimensional statistics literature to the \textit{process} level by considering the event
\begin{equation}\label{eq:lasso_valid_event}
\mathfrak{G}(\{\lambda_{t}\}_{t=1}^{T})=\left\{\forall t \in [T]: 4t^{-1}\|\nabla \mathcal{L}_{t}(\theta_0)\|_{\infty}\le \lambda_{t}\right\}.
\end{equation}
Given the above event, Theorem \ref{thm:always_oracle} in Section \ref{sec:time_unif_LASSO_oracle} shows that it is possible to build an \textit{always valid} estimation error bound on the proposed online LASSO procedure. Therefore, an optimal design of $\{\lambda_{t}\}_{t=1}^{T}$ should be the one to ensure that $\mathfrak{G}(\{\lambda_{t}\}_{t=1}^{T})$ holds with high probability.

Toward finding such an optimal selection, for a given confidence budget $\alpha \in (0,1)$, our goal is to find a regularization level sequence $\{\lambda_{t}(\alpha)\}_{t=1}^{T}$ that satisfies 
\begin{equation}\label{eq:good_event_pb}
    \mathbb{P}_{\theta_0}\left(\mathfrak{G}(\{\lambda_{t}(\alpha)\}_{t=1}^{T})\right) \ge 1-\alpha.
\end{equation}
As supported by Lemma \ref{lm_1:score_bouund} in Section \ref{subsec:formal_oolasso}, the proposed optimistic online regularization scheme (Definition \ref{def:opt_online_reg})
satisfies the property \eqref{eq:good_event_pb}. Therefore, when the agent learns the target demand parameter $\theta_0$ by solving the LASSO problem in \eqref{eq:current_LASSO_program} with the specified optimistic online regularization scheme in \eqref{eq:main_online_regularization}, the resulting estimator process $\{\widehat{\theta}_1, \widehat{\theta}_2, \cdots, \widehat{\theta}_T\}$
enjoys an \textit{always-validity}, i.e., the implemented online statistical learning procedure is theoretically valid at each decision point with a time-uniform estimation error bound (Theorem \ref{thm:always_oracle}). Such always-validity serves as a warranty on the robustness and safety of dynamic pricing algorithm design and fulfills the second desiderata pursued in Section \ref{sec:good_pricing_policy}.

\subsubsection{Formal  design}\label{subsec:formal_oolasso}

Here, we give a formal derivation of the online regularization scheme design to implement the principle outlined in Section \ref{subsec:basic_design_regu}. To analyze the event of valid Lasso procedure \eqref{eq:lasso_valid_event}, we first show a consequence of optimistic online regularization scheme \eqref{eq:main_online_regularization} on the infinity norm of score function process: 
\begin{lemma}{(Always Valid Score Function Process Bound)} 
\label{lm_1:score_bouund}
Under the optimistic online regularization scheme \eqref{eq:main_online_regularization},
it holds with probability at least $1-\alpha$ that 
\begin{equation}
\forall t \in 
[T]: \|\nabla \mathcal{L}_{t}(\theta_0)\|_{\infty} \le u_{W} \sqrt{ 2 t^{-1} \|\text{diag}(\widehat{\Sigma}_{[t]})\|_{\infty}  \ln(2d/\alpha)}.
\end{equation}
\end{lemma}

Lemma \ref{lm_1:score_bouund} provides a time-uniform control on the score function process $\{\nabla \mathcal{L}_{t}(\theta_0)\}_{t=1}^{T}$ of their infinity norm process. Concretely, the result 
bounds the fluctuation of score function process $\{\|\nabla\mathcal{L}_{t}(\theta_0)\|_{\infty}\}_{t=1}^{T}$ at the true demand parameter $\theta_0$ by carefully designing the online regularization sequence $\{\lambda_{t}\}_{t=1}^{T}$ to adaptive realized online uncertainty at each decision point. 
As remarked in Section \ref{sec:oolasso}, the online regularization scheme \eqref{eq:main_online_regularization} warrants always-validity of the \texttt{OOLASSO} procedure. Consequently, the design of optimistic online regularization scheme \eqref{eq:main_online_regularization} follows from Lemma \ref{lm_1:score_bouund} and the event of valid Lasso procedure \eqref{eq:lasso_valid_event}.


\begin{remark}
An advantage of the always-valid type result in Lemma \ref{lm_1:score_bouund} is that it holds for not only a constant decision horizon $T$ (independent from the pricing process) but also a \textit{random} decision horizon $T(w)$ (dependent on the pricing process). This property enables us to do valid inferences at randomly stopped times. 
\end{remark}

\subsubsection{Exploration-exploitation trade-off}
\label{sec:EEtradeoff_regu_online} 
We briefly discuss how the proposed optimistic online regularization scheme \eqref{eq:main_online_regularization} balances the explore-exploit trade-off during the pricing process. As we will show in Theorems \ref{thm:always_oracle} and \ref{thm:regret_analysis}, the revenue loss of the \texttt{OORMLP} in each decision point $t$ is of the same order as the squared estimation error bound $\|\widehat{\theta}_{t}-\theta_0\|_2^2$ which is bounded by $\lambda_{t}^2$.  
Thus, the regularization level $\lambda_{t}$ determines the {pricing optimism} of \texttt{OORMLP}. Price with larger revenue loss can be viewed as ``price exploration'' since larger price uncertainty helps the learning of $\theta_0$. On the other hand, a price with a smaller revenue loss can be viewed as ``price exploitation'', indicating that the agent exploits the learned demand parameter to maximize the collected revenue.

In general, the proposed optimistic online regularization scheme \eqref{eq:main_online_regularization} delivers a pricing policy that gradually shifts from price exploration to price exploitation. There are three main factors that contributed to pricing optimism: market noise knowledge $u_{W}$, product context process $\widehat{\Sigma}_{[t]}$, and confidence budge $\alpha$. Each of them captures different uncertainties happening in dynamic pricing, where $u_{W}$ measures demand uncertainty,  $\widehat{\Sigma}_{[t]}$ measures product feature uncertainty and $\alpha$ measures online procedure uncertainty. Section \ref{sec:oolasso} explains how these factors contribute to the regularization level in the face of online uncertainty. Section \ref{sec:simulation} investigates how these factors contribute to pricing optimism in the numerical experiments.

\subsection{Time-uniform lasso oracle inequality}\label{sec:time_unif_LASSO_oracle} 

This section establishes the time-uniform lasso oracle inequality (Theorem \ref{thm:always_oracle}) as a formal guarantee of the always-validity (the second desiderata; Section \ref{subsec:always_valid_bound}) of our framework.

To derive an error envelop for the estimates $\{\widehat{\theta}_{t}\}_{t=1}^{T}$ produced from \texttt{OOLASSO}, we first define a restricted eigenvalue process condition as a process analogue of a standard requirement in high-dimensional statistical estimation \citep{wainwright2019high}.

\begin{definition}\label{def:condi_RE_process}
For a product context process $\{x_{t}\}_{t=1}^{T}$, we say it satisfies a \textbf{restricted eigenvalue process} condition if there exists a sequence of positive number $\{\phi^2_t\}_{t=1}^{T}$ such that 
\begin{equation}\label{condi_RE_process}
\forall t \in [T] :\min_{J \subseteq [d]; |J|\le s_0} \min_{v\neq 0; \|v_{J^{c}}\|_{1}\le 3 \|v_{J}\|_{1}} \left(v^\top \widehat{\Sigma}_{[t]} v\right)/\|v\|_2^2 \ge \phi^2_t,
\end{equation}
where $v_{J}$ is the vector obtained by setting the elements of $v$ that are not in $J$ to zero.
\end{definition}

\begin{remark}\label{rmk:require_x_t}(On the requirement of product context sequence $\{x_{t}\}_{t=1}^{T}$)
Here, we only present the widely adopted restricted eigenvalue condition on the product context sequence $\{x_{t}\}_{t=1}^{T}$ to prove the time-uniform oracle inequality. Such conditions on the product context sequence can be relaxed by adapting arguments in high-dimensional inference literature (See, for example, \cite{chichignoud2016practical}).
\end{remark}

\begin{remark}(On the lower bound sequence $\{\phi^2_t\}_{t=1}^{T}$)
Let $\Sigma_0$ be the population covariance matrix of product context $x_{t}$ and denote its restricted eigenvalue as $\phi^2(\Sigma_0, s_0 )$. Based on matrix martingale concentration arguments, it can be shown that a choice of the lower bound sequence $\{\phi^2_t\}_{t=1}^{T}$ under confidence budget $\alpha$ is
\begin{equation*}
\phi_t^2
=
\phi^2(\Sigma_0, s_0 )
-
32s_0
\left[
\sqrt{2t^{-1}\ln(d(d+1)/2\alpha)}
+
t^{-1}\ln(d(d+1)/2\alpha)
\right].
\end{equation*}
\end{remark}



\begin{theorem}{(Always valid estimation error bound process)}\label{thm:always_oracle}
Suppose the product contexts process $\{x_{t}\}_{t=1}^{T}$ satisfies the restricted eigenvalue condition \eqref{condi_RE_process} with a non-random sequence $\{\phi_{t}^2\}_{t=1}^{T}$.  Then, under the online regularization scheme \eqref{eq:main_online_regularization}, it holds that:
\begin{equation}\label{eq:time-uniform_oracle}
\mathbb{P}_{\theta_0}\bigg(
\exists t \in 
[T]: \left\|\widehat{\theta}_{t} - \theta_0\right\|^2_2 \ge \frac{16s_0\lambda_{t}^2(\alpha)}{l_{W}^2\phi_{t}^2}\bigg) \le \alpha.
\end{equation}
\end{theorem} 

Theorem \ref{thm:always_oracle} provides a formal guarantee of the always-validity of our online statistical learning framework.
With a such guarantee, the user is allowed to terminate the dynamic pricing algorithm whenever they wish, and the result of estimation error bound maintains statistical validity. In particular, Theorem \ref{thm:always_oracle} indicates that the convergence rate of learning demand parameter $\theta_0$ is determined by three primary factors: (1) Non-smoothness of martingale difference noise conditional distribution function $F_{\eta_{t}|\mathcal{H}_{t-1}}$. This is captured by the minimal flatness defined by \eqref{eq:flatness}.  It controls the amount of information about the mean market value $\langle x_{t} ,\theta_0\rangle$ of product $x_{t}$ at each time step $t$. (2) The rate at which the product context $x_{t}$ explores the parameter space. This is governed by the restricted eigenvalue process condition (Definition \ref{def:condi_RE_process}). If the lower bound sequence $\{\phi_{t}^{2}\}_{t=1}^{T}$ is small, the product context is relatively aligned and one requires a larger sample size to estimate the demand parameter within a specified accuracy. (3) The complexity of demand parameter $\theta_0$. This is captured through the sparsity measure $s_0$ in the feasible parameter space \eqref{eq:feasible_parameter}.

\subsection{Regret analysis of the \texttt{OORMLP} algorithm}\label{subsec:regret_ana}
\vspace{-2mm}

This section establishes the regret analysis (Theorem \ref{thm:regret_analysis}) of the proposed \texttt{OORMLP} dynamic pricing algorithm (Algorithm \ref{alg:OORMMLP}) as a formal guarantee of the revenue-maximization quality (the third desiderata; Section \ref{subsec:regret}) of our online statistical learning framework. 
The following theorem bounds the regret of the proposed \texttt{OORMLP} dynamic pricing algorithm.
\begin{theorem}{(Regret guarantee for \texttt{OORMLP} algorithm)}\label{thm:regret_analysis}
Suppose the product context sequence $\{x_{t}\}_{t=1}^{T}$ satisfies the restricted eigenvalue condition \eqref{condi_RE_process} with a non-random sequence $\{\phi_{t}^2\}_{t=1}^{T}$.  Then, under the online regularization scheme \eqref{eq:main_online_regularization}, with probability at least $1-\alpha$,
\begin{equation}\label{eq:FullRegret}
\textbf{Regret}_{\texttt{OORMLP}}(T)  \leq \frac{256C s_0u_{W}^2}{l_{W}^2\min_{t\in[T]}\phi_{t}^2}\ln(\frac{2d}{\alpha})\log T.
\end{equation}
\end{theorem}

To read the regret bound \eqref{eq:FullRegret}, we break it into three elements of dynamic pricing problems. First, the regret bound depends on the product market value model (Sec. \ref{sec:choice_model}) in terms of $s_0$, the sparsity level of demand coefficient, and $d$, the dimension of product context, at the rate $s_0\log d$. Second, the regret bound depends on the martingale difference noise process $\{\eta_{t}\}_{t=1}^{T}$ at \eqref{eq:Will_to_Pay_linear_model} in terms of $u_{W}$, the maximal steepness and $l_{W}$, the minimal flatness, at the rate $(u_{W}/l_{W})^2$. Third, the regret bound depends on the product context sequence $\{x_{t}\}_{t=1}^{T}$ via the restricted eigenvalue sequence $\{\phi^2_t\}_{t=1}^{T}$ (Definition \ref{condi_RE_process}) at the rate $1/\min_{t\in [T]}\phi_{t}^2$. Under additional assumptions on the boundedness of these parameters, we achieve an $O(\log T)$ regret bound of \texttt{OORMLP}, which meets the information-theoretical lower bound shown in \cite{javanmard2019dynamic}. 

\begin{remark}{(Comparison to \texttt{RMLP} algorithm proposed in \cite{javanmard2019dynamic})} \label{rm:alg_design}
We emphasize that our regret bound (Theorem \ref{thm:regret_analysis}) is always valid in the sense that the result holds for a random decision horizon $T$. In contrast, the regret bound of \texttt{RMLP} only holds for a fixed constant decision horizon $T$. This is because  
the \texttt{RMLP} algorithm used the \textit{doubling trick} to apply batch-type concentration result based on i.i.d. noise assumption in dynamic pricing algorithm design, while our result is based on martingale concentration. First, \texttt{RMLP} is not as sample efficient as \texttt{OORMLP}. This is because \texttt{RMLP} needs to reset the algorithm several times during the pricing process to achieve logarithm regret. On the other hand, our \texttt{OORMLP} uses a novel non-asymptotic martingale concentration to avoid resetting the algorithm during the whole pricing process and still achieves logarithm regret. Second, \texttt{RMLP} relies on an i.i.d. noise assumption, while \texttt{OORMLP} allows for a more flexible martingale difference noise. As shown in experiments in Section \ref{sec:simulation}, \texttt{OORMLP} is more sample efficient and robust to noise assumptions.
\end{remark} 

\section{Experiments}\label{sec:simulation}

We evaluate the performance of the proposed \texttt{OORMLP} algorithm on both synthetic and real-world data. Additional simulations with dependent context sequence, sensitivity tests are provided in the supplement. 

\subsection{Simulations with independent context sequence}
\label{subsec:synt_data}

\begin{figure}[t!]
\centering
\includegraphics[width=\textwidth]{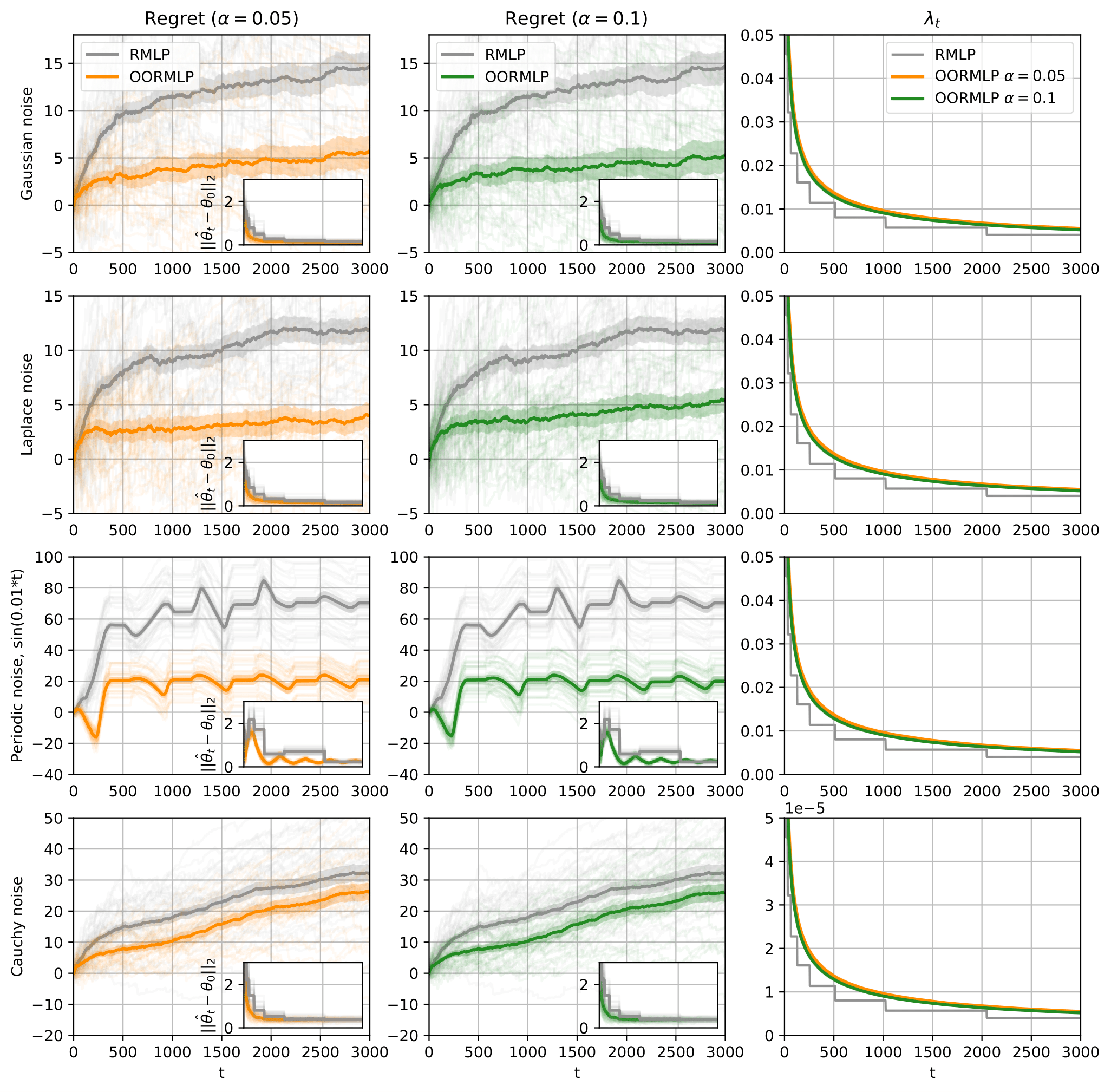}
\caption{Comparison between \texttt{RMLP} and \texttt{OORMLP} when $d=10$. 
         \textbf{First row:} $\eta_t \sim N(0,1)$. 
         \textbf{Second row:} $\eta_t \sim \text{Laplace}(0,1)$. 
         \textbf{Third row:} 
         $\eta_t = \sin(0.01 t)$.
         \textbf{Fourth row:} $\eta_t \sim \text{Cauchy}(0,1)$. 
         \textbf{Two columns on the left:} different choices of confidence budget $\alpha$. 
         \textbf{Rightmost column:} $\lambda_t$ for the experiments.
         \textbf{Small figures in each subfigure:} Estimation error $\|\hat\theta_t - \theta_0\|_2$. Each transparent line represents one experiment. The solid lines and error bars represent the sample mean and its standard deviation.} 
\label{fig:all_three_noises}
\end{figure}

We compare \texttt{OORMLP} with \texttt{RMLP} under four representative demand uncertainty settings: (i)  Gaussian ($\eta_t \sim N(0,1)$) (ii) Laplace ($\eta_t \sim \text{Laplace}(0,1)$) (iii) Periodic ($\eta_t = \sin(0.01 t)$) and (iv) Cauchy ($\eta_t \sim \text{Cauchy}(0,1)$). Settings (i) and (ii) stand for instances of log-concave distributions, where (ii) has a heavier tail than (i). Setting (iii) stands for an instance of time-series noise, where the noises between two adjacent time points are strongly dependent. Setting (iv) stands for distribution beyond the log-concave distribution assumed in our theoretical analysis. This setting investigates our algorithm under model misspecification. 
We set $\theta_0=(1,1,1,0,0,0,0,0,0,0)$ with $d=10$. Each entry in the product context vector $x_t \in \mathbb R^{10}$ is generated from $N(0,1)$ and truncated to $[-1, 1]$ (Similar synthetic data generation procedure is implemented by \cite{bastani2020online}). Therefore,  $\|x_t\|_\infty \leq 1$.

We implement our \texttt{OORMLP} algorithm at two confidence budgets ($\alpha= 0.05$ and $0.1$) which refer to different levels of pricing optimism, 
and compare our results with \texttt{RMLP} in \cite{javanmard2019dynamic}.
In real scenarios, we do not know the exact distribution of demand uncertainty in advance, and hence we design the pricing function $g(\cdot)$ by assuming the uncertainty is standard normal ($\eta_{t}\sim N(0,1)$). Such consideration tests the robustness of our algorithm when the demand uncertainty is unknown. Since $\|\theta_0\|_1=3$, we set $W=10$ for both \texttt{OORMLP} and \texttt{RMLP}. In practice, the theoretical online regularization choice in \eqref{eq:main_online_regularization} might be conservative. To compare the finite-time performance of \texttt{OORMLP} and \texttt{RMLP}, we scale the regularization sequence $\{\lambda_t\}_{t=1}^T$ of both methods by the same scaling parameter $c_{\lambda}=0.001$ (except for the Cauchy  noise setting where we use $c_{\lambda}=10^{-6}$ for both methods). 
We compute the mean and confidence interval of regrets over 32 replications. Figure \ref{fig:all_three_noises} reports the results for the regret, the estimation error, and the regularization sequence used, which show the superiority and robustness of our algorithm (See Remark \ref{rm:compare_RMLP}). 


Below we give general remarks and rationales of our \texttt{OORMLP} from the perspectives of variance control, sample efficiency, and regret reduction.
\begin{enumerate}[leftmargin=1mm]
    \item \textbf{Sample efficiency on estimation error process.} Small figures in each subfigure at Figure \ref{fig:all_three_noises} visualize the estimator error process of \texttt{RMLP} and \texttt{OORMLP}. In the first three uncertainty settings, \texttt{OORMLP} achieves smaller estimation errors than \texttt{RMLP}. This aligns with Remark \ref{rm:alg_design} that \texttt{OORMLP} is more sample efficient than \texttt{RMLP} since it avoids resetting the algorithm. Remarkably, the estimator accuracy of \texttt{RMLP} is especially fragile in the setting (iii) of periodic noise. This is because \texttt{RMLP} uses samples only from previous episodes and updates geometrically, and its estimation accuracy and pricing performance are impeded in a scenario where noises between two adjacent time points are strongly dependent. In contrast to \texttt{RMLP}, our \texttt{OORMLP} enjoys a superior design in terms of sample efficiency and robustness in such periodic noise settings. Finally, in setting (iv) of Cauchy noise which violates our log-concave noise assumption, \texttt{OORMLP} performs similarly to \texttt{RMLP}. 
    \item \textbf{Confidence budget and regret reduction.} Similar to the performance in the estimation error process, \texttt{OORMLP} achieves much smaller regrets than \texttt{RMLP} in the first three uncertainty settings.
    The first two columns in the first row of Figure \ref{fig:all_three_noises} show an interesting phenomenon that a larger confidence budget $\alpha$
    leads to a more substantial regret reduction of our \texttt{OORMLP}, while the performance of \texttt{RMLP} is not adaptive to $\alpha$. This aligns with our discussion in Section \ref{sec:oolasso} on how \texttt{OORMLP} balances the explore-exploit trade-off during the pricing process. 
\item \textbf{Shape of online regularization scheme.} The rightmost column of Figure \ref{fig:all_three_noises} visualizes how non-asymptotic martingale concentration arguments authorize a process-level online regularization scheme. Compared to \texttt{RMLP} which resets itself geometrically (when $t = 2^{k}, k \in \mathbb{N}$) without considering product feature uncertainty, our \texttt{OORMLP} delivers a smooth regularization process against both product context uncertainty and demand uncertainty.
    
\end{enumerate}

\begin{figure}[h!]
 \includegraphics[width=1.0\textwidth]{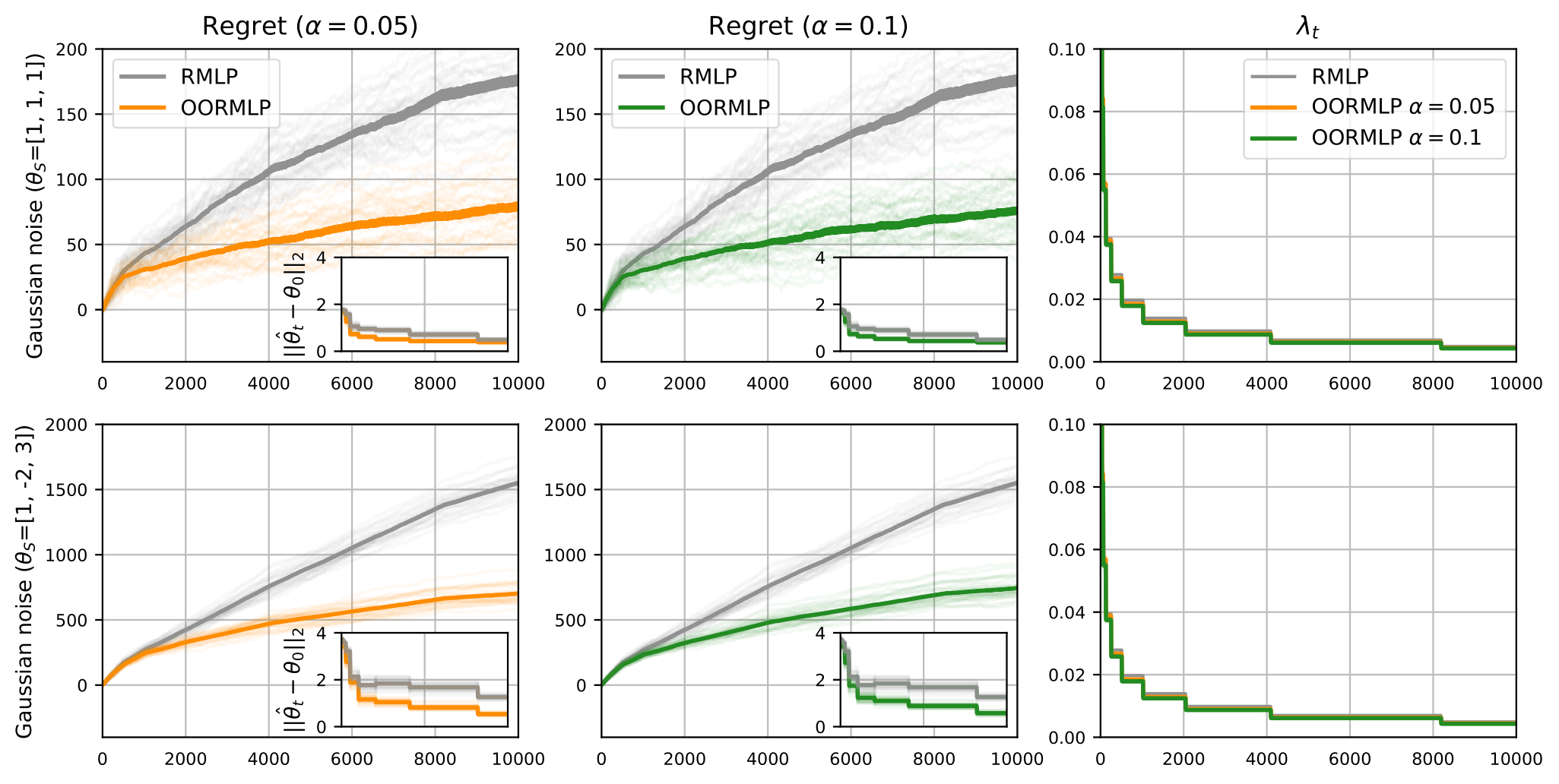}
  \caption{Comparison between \texttt{RMLP} and \texttt{OORMLP} when $d=1000$. $\eta_t \sim N(0,1)$. We use $\theta_S$ to denote the values in the support (the only non-zero entries) of $\theta_0$.
         \textbf{First row:} $\theta_S = (1,1,1)$. 
         \textbf{Second row:} $\theta_S = (1,-2,3)$. 
         Details for subfigures are the same as in Figure \ref{fig:all_three_noises}. 
         } 
  \label{fig:high_dim_simulation}
\end{figure}

For high-dimensional experiments, we set $d=1000$ and use the Gaussian noise setting  ($\eta_t \sim N(0,1)$) again. We consider two settings of the true demand parameter: $\theta_0 = (1,1,1,0,0,\cdots,0)$ and $\theta_0 = (1,-2,3,0,0,\cdots,0)$.
We use $c_{\lambda}=0.001$ and set $W=10$.  Here to save computation resources, in both this high-dimensional setting and the real data setting below, we update the estimation of \texttt{OORMLP} only at $t=2^k, k\in\mathbb{N}$ as in \texttt{RMLP}. 
Figure \ref{fig:high_dim_simulation} shows the results of $t\in[0,10000]$ over 32 replicates.
\texttt{OORMLP} performs better than \texttt{RMLP} even with the same number of estimation updates.
This regret reduction mainly comes from the larger sample size used by \texttt{OORMLP}. As mentioned in Remark \ref{rm:alg_design}, \texttt{RMLP} used a doubling trick to apply batch-type concentration results, while our result for \texttt{OORMLP} is based on a martingale concentration. Therefore, \texttt{RMLP} only updates its estimate using the most recent batch while \texttt{OORMLP} updates its estimate using all historical information. This means the sample size used by \texttt{OORMLP} is twice larger than that used in \texttt{RMLP}, and hence \texttt{OORMLP} is more sample efficient.

\subsection{Real data analysis on auto loan applications}
\label{subsec:real_data}

We demonstrate the efficiency of \texttt{OORMLP} in setting personalized lending rates for an online auto loan company in the United States. Personalization of prices in the lending industry is widely used and well-accepted. Our experiments are based on a real-life data set \textit{CPRM-12-001: On-Line Auto Lending} provided by the Center for Pricing and Revenue Management at Columbia University.
This database contains data on all 208,805 auto loan applications received by a major online lender in the United States between July 2002 and November 2004. The data collection contains the date on which prospective borrowers submitted an application, the sort of loan they requested (term and amount), and some personal information. Additionally, the data collection includes whether the online lender authorized the application, the annual percentage rate (APR) given and whether a contract was executed. In this context, clients' demand responses are binary, indicating whether or not a loan was agreed upon. This dataset was studied in many dynamic pricing literatures, e.g., \cite{phillips2015effectiveness}, \cite{ban2021personalized}, \cite{bastani2021meta}.  

A summary of the data set (with descriptive statistics on the demand and available features) is shown in the Table 3 in \cite{ban2021personalized}. 
The column ``apply" is the binary demand indicator for eventual contract and is the response variable with value in $\{0,1\}$ for the market value model. 
There are 18 feature variables, both discrete categorical (e.g., type of financing, type of car, customer state) and continuous (e.g., FICO score, customer rate, competitor's rate). We prepossess the categorical variable to dummy variables and normalize the continuous variable to values with mean 0 and maximum absolute value 1.

This pricing problem is a special instance of the problem formulation in Section \ref{sec:choice_model}, with demand being a binary variable. In this situation, the price of a loan is determined by subtracting the loan amount from the net present value of future payments. Formally, we can calculate the price from the other variables in the dataset through
$p=\mathrm{Monthly}\ \mathrm{Payment} \times \sum_{\tau=1}^{\mathrm{Term}} (1+\mathrm{Rate})^{-\tau} - \mathrm{Loan}\ \mathrm{Amount}.$
Here, we use one thousand dollars as a basic unit for the price $p$. Also, note that the dimension of the variables in this dataset is $d=71$ since we construct dummy variables from the categorical variables.

In practice, it is hard to retrieve real-time feedback from clients on any dynamic pricing strategy until the pricing policy has been implemented in the data collection system. Thus, we apply off-policy learning used in \cite{ban2021personalized} to estimate the customer choice model using $\widehat{\theta} \equiv \arg\min \mathcal{L}(\theta)$ where $\mathcal{L}(\theta)$ is defined by (\ref{eq:self_info}) but across the entire dataset with the assumption $\eta_t$ being i.i.d. following $N(0,1)$. This optimization problem is the same as (\ref{eq:current_LASSO_program}) with $\lambda_t=0$ and $W=\infty$. We use (\ref{eq:sale_status_stoc_model}) with $\theta_0 = \widehat{\theta}$ as the ground truth model for generating the response of each consumer given any price. More specifically, to generate data from this model, we sample the covariates $x_t$ from the original dataset and $\eta_t$ from $N(0,1)$, then we calculate the market value $v_t$ and the response $y_t$ using (\ref{eq:Will_to_Pay_linear_model}) and (\ref{eq:sale_status_stoc_model}).

Similar to the simulation study in Section \ref{subsec:synt_data}, we design the pricing function $g(\cdot)$ by assuming the uncertainty is standard normal ($\eta_{t}\sim N(0,1)$). Since the ground truth model has $\|\theta_0\|_1 = 33.68$, we use $W=100$ as the upper bound for $\|\theta\|_1$ in the online estimation of $\theta$ for both \texttt{OORMLP} and \texttt{RMLP}. The scaling parameter is  $c_{\lambda}=0.00001$ and we update the estimation of $\hat\theta_t$ at $t=2^k, k\in\mathbb{N}$. We compare \texttt{OORMLP} to \texttt{RMLP} using experiment with $t\in[0,5000]$ over 32 replicates. Figure \ref{fig:autoloan_data} reports the results for the regret, the estimation error, and the regularization sequence used. 

\begin{figure}[h!]
\centering
\includegraphics[width=1.0\textwidth]{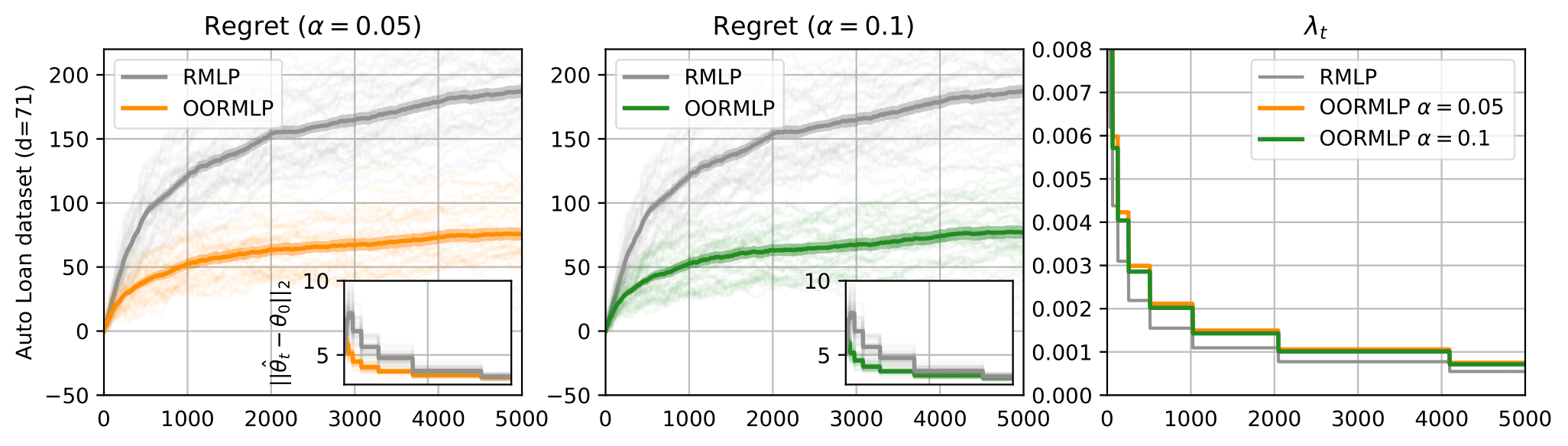}
\caption{Comparison between \texttt{RMLP} and \texttt{OORMLP} on the On-Line Auto Lending dataset. Details for subfigures are the same as in Figure \ref{fig:all_three_noises}.} 
\label{fig:autoloan_data}
\end{figure}

While both \texttt{OORMLP} and \texttt{RMLP} enjoy sublinear growth of regret, \texttt{OORMLP} obtains more accurate and stable estimation of $\theta$ and much less regret than \texttt{RMLP} at $T=5000$ time periods across all confidence budgets under similar regularization sequence $\{\lambda_t\}_{t=1}^{5000}$. This is consistent to our observation on the comparison results with synthetic data.
These results support that \texttt{OORMLP} enjoys substantial further regret reduction compared to \texttt{RMLP} and supports the claimed superiority of the proposed online regularization scheme.

\section*{Acknowledgment}
The authors thank the editor Professor Jane-Ling Wang, the associate editor and two anonymous reviewers for their valuable comments and suggestions which led to a much improved paper. Will Wei Sun's research was partially supported by NSF-SES grant (2217440). Guang Cheng's research was partially supported by ONR grant (N00014-22-1-2680) and NSF-SCALE MoDL grant (2134209). Any opinions, findings, and conclusions expressed in this material are those of the authors and do not reflect the views of the Office of Naval Research or the National Science Foundation. The authors report there are no competing interests to declare.


\baselineskip=13pt
\bibliographystyle{ims}
\nocite{*}
\bibliography{ref}

\end{document}